\documentclass[11pt]{article}

\usepackage[T1]{fontenc}
\usepackage[utf8]{inputenc}
\usepackage{lmodern}
\usepackage[a4paper,margin=1in]{geometry}
\usepackage{amsmath,amssymb,amsthm,mathtools}
\usepackage{microtype}
\usepackage{hyperref}
\usepackage{enumitem}
\DeclareMathOperator*{\argmax}{arg\,max}
\hypersetup{
  colorlinks=true,
  linkcolor=blue,
  citecolor=blue,
  urlcolor=blue
}

\newtheorem{theorem}{Theorem}
\newtheorem{corollary}{Corollary}

\newtheorem{remark}{Remark}
\newtheorem{definition}{Definition}

\title{Hidden Biases in Conditioning Autoregressive Models}
\author{
Fran\c{c}ois Pachet\\
\small LIP6, Sorbonne Universit\'e\\
\small \texttt{pachet@gmail.com}\\
\and
Pierre Roy\\
\small Soundtrap\\
\small \texttt{roypie@gmail.com}\\
}
\date{}

\begin{document}

\maketitle

\begin{abstract}
Large language and music models are increasingly used for constrained generation: rhyming lines, fixed meter, inpainting or infilling, positional endings, and other global form requirements. These systems often perform strikingly well, but the induced procedures are usually not exact conditioning of the underlying autoregressive model. This creates a hidden inferential bias, distinct from the better-known notion of bias inherited from the training set: samples are distorted relative to the true constrained distribution, with no generic guarantee of complete coverage of the admissible solution space or of correct conditional probabilities over valid completions. We formalize several exact inference tasks for autoregressive models and prove corresponding hardness results. For succinctly represented autoregressive models whose next-token probabilities are computable in polynomial time, exact sentence-level maximum a posteriori (MAP) decoding is NP-hard. This hardness persists under unary and metrical constraints. On the sampling side, exact conditioned normalization is \#P-hard even for regular constraints such as fixed-length terminal events. Unlike finite-state Markov models, general autoregressive models do not admit a bounded-state dynamic program for these tasks. These results formalize a standard claim in the neural decoding literature: local autoregressive sampling is easy, whereas exact decoding and exact conditioning under global form constraints are computationally intractable in general.
\end{abstract}

\section{Introduction}

Large language models are now routinely used for constrained generation. Users ask them to write rhyming lines, continue a poem in a prescribed meter, fill in missing text, complete a partially specified sentence, enforce a suffix or a letter pattern at a given position, or satisfy combinations of such requirements. Similar uses appear in symbolic music, where sequence models are asked to inpaint missing passages or complete a score subject to formal constraints. Empirically, these systems often perform strikingly well. They produce many valid-looking constrained outputs, and they often give the impression that formal control is largely solved.

That impression is misleading. When an autoregressive model is used in this way, the resulting procedure is usually not exact conditioning of the original model on the requested constraint. Instead, practical systems rely on heuristic search, reranking, local reweighting, rejection sampling, or separately trained infilling architectures. The resulting samples are therefore biased relative to the exact conditional law of the fixed underlying autoregressive model. More importantly, one generally does not know which part of the admissible solution space is actually explored, whether all valid completions remain reachable, or whether feasible completions are produced with the correct conditional probabilities.

This is a different notion of bias from the one usually discussed for large language models. There the term typically refers to skew inherited from the training distribution or reflected in the generated content \cite{blodgett-etal-2020-language}. Here the bias is computational: with the pretrained model held fixed, approximate constrained-generation procedures distort the exact conditional distribution induced by the requested constraint.

The results of this paper explain why. The issue is not that such constraints are linguistically or musically exotic. On the contrary, many of them are simple to state and even regular or additive in structure. The difficulty comes from the mismatch between local left-to-right generation and global exact satisfaction of a property of the whole completed sequence.

An autoregressive model with parameters $\theta$ defines, for each finite sequence $x=x_{1:n}$,
\[
P_{\theta}(x_{1:n}) = \prod_{i=1}^{n} P_{\theta}(x_i \mid x_{1:i-1}).
\]
This factorization suggests computational simplicity. For unconstrained generation, that is correct: one samples sequentially from the next-token conditionals. But sequence-level optimization or conditioning is a different problem. Given access to the local conditionals, one may ask for the globally most probable complete sequence,
\[
x^* \in \argmax_{x \in \mathcal{C}} P_{\theta}(x),
\]
or for exact sampling from the model conditioned on some global event.

For finite-state Markov models, such tasks are solved by dynamic programming because the relevant dependence on the past is summarized by a bounded state. General autoregressive models do not enjoy this property. Their next-token distribution may depend on the entire prefix through an unbounded contextual representation. As a result, the factorization of the probability does not imply a Viterbi-style algorithm for recovering the most probable sequence, nor an efficient forward-backward style algorithm for exact conditioning.

This distinction is reflected in practice. In the neural sequence generation literature, exact MAP decoding is treated as computationally difficult, and approximate methods such as greedy search, beam search, and reranking are used instead \cite{stahlberg-byrne-2019,meister-vieira-cotterell-2020,eikema-aziz-2020}. Yet this point is usually made algorithmically rather than complexity-theoretically. The purpose of this note is to make a family of precise hardness statements explicit.

Our contribution is elementary and conceptually clarifying. We prove that exact sentence-level MAP decoding is NP-hard for a natural class of succinctly represented autoregressive models. The hardness persists under unary constraints and under metrical constraints. On the sampling side, we prove that exact conditioned normalization is already \#P-hard for regular constraints as simple as ``having fixed length $L$ and ending in \texttt{eos}.'' This perspective explains why current constrained-generation systems can be useful and impressive while still failing to deliver exact conditional sampling: they do not provide a generic guarantee of complete coverage of the valid solution space, nor of correct conditional probabilities over that space.

\section{Setup}

Let $\Sigma$ be a finite vocabulary containing a distinguished end-of-sequence symbol $\mathtt{eos}$. An autoregressive model assigns probabilities to finite complete sequences through conditionals of the form
\[
P_{\theta}(x_{1:n}) = \prod_{i=1}^{n} P_{\theta}(x_i \mid x_{1:i-1}),
\]
with the convention that generation terminates as soon as $\mathtt{eos}$ is produced. We therefore write $\mathcal{C}$ for the set of complete sequences
\[
\mathcal{C}
=
\bigcup_{n > 1} \left\{
                            x_{1:n}\in\Sigma^n : x_n=\mathtt{eos}
                            \text{ and }
                            x_i\neq \mathtt{eos} \text{ for } i<n
                            \right\}.
\]

\begin{definition}[Sentence-level MAP decoding]
Given an autoregressive model $P_{\theta}$, the sentence-level maximum a posteriori (MAP) problem is
\[
x^* \in \argmax_{x \in \mathcal{C}} P_{\theta}(x).
\]
\end{definition}

In the present setting it simply denotes the most probable complete sequence under the model.

To obtain a genuine complexity statement, the model itself must be part of the input. We therefore consider a class $\mathcal{M}$ of \emph{succinctly represented autoregressive models} such that:

\begin{enumerate}[label=(\roman*)]
\item each model $P_{\theta} \in \mathcal{M}$ has a finite description whose size is polynomial in the instance size;
\item for every prefix $x_{1:i-1}$ and every token $a \in \Sigma$, the conditional probability $P_{\theta}(a \mid x_{1:i-1})$ is a rational number whose exact value is computable in polynomial time from the model description and the prefix.
\end{enumerate}

This captures the natural setting in which local next-token evaluation is efficient, while global optimization or exact conditioning over complete sequences can still be hard.

\section{Exact MAP decoding}

\begin{theorem}\label{thm:map-hard}
Exact sentence-level MAP decoding over succinctly represented autoregressive models is NP-hard.
\end{theorem}

\begin{proof}
We reduce from SAT. Let $\varphi$ be a CNF formula over Boolean variables $v_1,\dots,v_m$. From $\varphi$ we construct, in polynomial time, an autoregressive model $P_{\varphi}$ over the alphabet
\[
\Sigma = \{0,1,b_0,b_1,\mathtt{eos}\}.
\]
The constructed model assigns higher probability to satisfying assignments than to non-satisfying ones, so that any MAP solution encodes a satisfying assignment.

For the first $m$ positions, the model chooses bits uniformly:
\[
P_{\varphi}(x_i=0 \mid x_{1:i-1}) = P_{\varphi}(x_i=1 \mid x_{1:i-1}) = \tfrac12,
\qquad i=1,\dots,m.
\]
Hence every prefix $a=(a_1,\dots,a_m)\in\{0,1\}^m$ has probability $2^{-m}$ and encodes a truth assignment for $\varphi$.

At step $m+1$, the model inspects the assignment encoded by the prefix. If $a \models \varphi$, it emits $\mathtt{eos}$ with probability $1$, so the sequence terminates at length $m+1$. Otherwise it emits $b_0$ and $b_1$ with probabilities $\tfrac12$ each. In the latter case, at step $m+2$ the model emits $\mathtt{eos}$ with probability $1$.

Therefore,
\[
P_{\varphi}(a\,\mathtt{eos}) = 2^{-m}
\qquad
\text{if } a \models \varphi,
\]
whereas for each $j\in\{0,1\}$,
\[
P_{\varphi}(a\,b_j\,\mathtt{eos}) = 2^{-m-1}
\qquad
\text{if } a \not\models \varphi.
\]

If $\varphi$ is satisfiable, every satisfying assignment yields a complete sequence of probability $2^{-m}$, and every complete sequence corresponding to a non-satisfying assignment has strictly smaller probability $2^{-m-1}$. Hence any MAP-optimal sequence encodes a satisfying assignment.

If $\varphi$ is unsatisfiable, then every complete sequence has probability exactly $2^{-m-1}$.

Thus an exact MAP decoder decides SAT: compute a most probable complete sequence and check whether its $(m+1)$st token is $\mathtt{eos}$, equivalently whether its first $m$ symbols satisfy $\varphi$. The construction is polynomial-size, and each conditional probability is polynomial-time computable because the only nontrivial step is evaluating whether an assignment satisfies a CNF formula. Since SAT is NP-hard, exact sentence-level MAP decoding is NP-hard.
\end{proof}

\section{Unary-constrained MAP decoding}

We now consider position-wise unary constraints. Let $S_1,\dots,S_n \subseteq \Sigma$ be allowed token sets, and consider the constrained optimization problem
\[
x^* \in \argmax_{x\in\mathcal{C},\ |x|=n} P_{\theta}(x)
\qquad
\text{subject to } x_i \in S_i \text{ for all } i.
\]

\begin{corollary}\label{cor:unary-hard}
Exact sentence-level MAP decoding under unary position-wise constraints is NP-hard.
\end{corollary}

\begin{proof}
Using the construction above, impose the unary constraints
\[
x_i \in \{0,1\}, \qquad i=1,\dots,m,
\]
and
\[
x_{m+1} = \mathtt{eos}.
\]
Then the feasible sequences are exactly those of the form $a\,\mathtt{eos}$ with $a\in\{0,1\}^m$. Under the model $P_{\varphi}$,
\[
P_{\varphi}(a\,\mathtt{eos})=
\begin{cases}
2^{-m}, & \text{if } a \models \varphi,\\[4pt]
0, & \text{otherwise},
\end{cases}
\]
because a non-satisfying assignment cannot emit $\mathtt{eos}$ at position $m+1$ in this construction. Hence the optimum constrained probability is positive if and only if $\varphi$ is satisfiable, and an exact constrained MAP decoder decides SAT by returning an optimal feasible sequence whose probability can then be evaluated in polynomial time. Therefore unary-constrained MAP decoding is NP-hard. In particular, NP-hardness already holds for the fixed-terminal constraint that the sequence end with $\mathtt{eos}$ at a prescribed position.
\end{proof}

\section{Exact conditioning on a terminal token at fixed length}\label{sec:fixedlength}

Constraining the last token to be $\mathtt{eos}$ at a prescribed length $L$ is a special case of unary constraints: one simply takes
\[
S_i=\Sigma\setminus\{\mathtt{eos}\} \quad (i<L),
\qquad
S_L=\{\mathtt{eos}\}.
\]
However, the associated inference task is different. For MAP decoding, this is just a special unary-constrained optimization problem. For exact sampling, what matters is the normalization constant of the conditioned distribution.

For fixed length $L$, define
\[
Z_L(\theta)
=
\sum_{x_{1:L-1}\in(\Sigma\setminus\{\mathtt{eos}\})^{L-1}} P_{\theta}(x_{1:L-1},\mathtt{eos}),
\]
that is, the total probability mass of all sequences of length exactly $L$ ending in $\mathtt{eos}$.

Exact sampling from the fixed-length conditional distribution
\[
P_{\theta}(x \mid x\in\mathcal{C},\ |x|=L)
\]
requires, at each step, access to conditioned continuation masses of this kind.

\begin{theorem}\label{thm:zl-hard}
For succinct autoregressive models with polynomial-time next-token evaluation, exact computation of $Z_L(\theta)$ is \#P-hard.
\end{theorem}

\begin{proof}
We reduce from \#SAT. Let $\varphi$ be a CNF formula over Boolean variables $v_1,\dots,v_m$, and let $L=m+1$. Construct an autoregressive model $P_{\varphi}$ over
\[
\Sigma = \{0,1,b_0,b_1,\mathtt{eos}\}
\]
as follows. For the first $m$ positions,
\[
P_{\varphi}(x_i=0 \mid x_{1:i-1}) = P_{\varphi}(x_i=1 \mid x_{1:i-1}) = \tfrac12,
\qquad i=1,\dots,m.
\]
Thus the prefix $x_{1:m}$ encodes a uniformly random truth assignment $a\in\{0,1\}^m$.

At step $m+1$, if $a\models\varphi$, the model emits $\mathtt{eos}$ with probability $1$; otherwise it emits $b_0$ and $b_1$ with probabilities $\tfrac12$ each, followed by $\mathtt{eos}$ with probability $1$ at step $m+2$.

Hence,
\[
Z_L(P_{\varphi})
=
\sum_{a\in\{0,1\}^m} P_{\varphi}(a,\mathtt{eos})
=
\frac{\#\mathrm{SAT}(\varphi)}{2^m}.
\]
Therefore exact computation of $Z_L(\theta)$ yields the number of satisfying assignments of $\varphi$. Since \#SAT is \#P-hard, so is exact computation of $Z_L(\theta)$.
\end{proof}

\begin{corollary}
Exact sampling from the fixed-length conditional distribution
\[
P_{\theta}(x \mid x\in\mathcal{C},\ |x|=L)
\]
does not reduce to ordinary ancestral sampling in the general autoregressive setting. It requires global conditioned continuation masses whose exact computation is \#P-hard.
\end{corollary}

\begin{remark}
The set
\[
R_L = \{x\in\mathcal{C} : |x|=L\}
=
(\Sigma\setminus\{\mathtt{eos}\})^{L-1}\mathtt{eos}
\]
is a regular language. Thus regularity of the constraint language is not sufficient for tractability in the general autoregressive setting considered here.
\end{remark}

\section{Metrical constraints}

A fixed terminal position is a special case of a metrical constraint. Let each token $a\in\Sigma$ be assigned a nonnegative integer syllabic weight $w(a)$, and define the set of metrical sentences of total weight $K$ by
\[
\mathcal{S}_K
=
\left\{
x\in\mathcal{C} : \sum_{i=1}^{|x|} w(x_i)=K
\right\}.
\]

The corresponding constrained MAP problem is
\[
\argmax_{x\in\mathcal{S}_K} P_{\theta}(x),
\]
and the corresponding normalization constant is
\[
Z_K(\theta)=\sum_{x\in\mathcal{S}_K} P_{\theta}(x).
\]

The fixed-length terminal constraint studied above is an immediate special case: assign unit weight to every token, that is, $w(a)=1$ for all $a\in\Sigma$. Then the condition $\sum_i w(x_i)=K$ is exactly the condition that the sentence has length $K$.

\begin{corollary}
Exact MAP decoding under a fixed metrical constraint is NP-hard.
\end{corollary}

\begin{proof}
Take $w(a)=1$ for all tokens. Then a metrical constraint with target $K=L$ is equivalent to the fixed-length terminal constraint at length $L$. By the fixed-terminal special case of Corollary~\ref{cor:unary-hard}, exact MAP decoding under this unit-weight metrical constraint is NP-hard.
\end{proof}

\begin{corollary}
For succinct autoregressive models with polynomial-time next-token evaluation, exact computation of the total probability mass of metrical sentences of syllable count $K$ is \#P-hard.
\end{corollary}

\begin{proof}
Again take $w(a)=1$ for all tokens. Then
\[
Z_K(\theta)=\sum_{x\in\mathcal{S}_K} P_{\theta}(x)
\]
is exactly the fixed-length terminal normalization constant $Z_L(\theta)$ with $L=K$. The claim therefore follows immediately from the preceding theorem.
\end{proof}

\begin{remark}
This reduction is deliberately simple: it shows that metrical generation is already hard even in the degenerate unit-syllable case. Reductions from subset-sum or counting subset-sum would match the additive nature of meter more directly.
\end{remark}

\section{Decision version}

For completeness, consider the threshold problem.

\begin{definition}[MAP-THRESHOLD]
\textbf{Input:} a succinct autoregressive model $P_{\theta}$, an integer $n$, and a rational threshold $\tau$.

\textbf{Question:} does there exist a complete sequence $x_{1:n}\in\mathcal{C}$ of length $n$ such that
\[
P_{\theta}(x_{1:n}) \ge \tau ?
\]
\end{definition}

\begin{corollary}
MAP-THRESHOLD is NP-complete.
\end{corollary}

\begin{proof}
Membership in NP is immediate: guess a sequence $x_{1:n}$, compute the $n$ local conditional probabilities, multiply them, and compare the result with $\tau$. NP-hardness follows from the SAT reduction above by taking $n=m+1$ and $\tau=2^{-m}$.
\end{proof}

\section{Tractable and intractable constraints}

The results above do not say that every constraint makes sampling or decoding hard. The correct distinction is between constraints that can be enforced from a bounded state and constraints whose exact enforcement requires nontrivial global continuation masses.

A first easy case is prefix conditioning. If one conditions on a fixed prefix,
\[
x_{1:k}=u,
\]
then exact sampling is immediate: one starts generation from the prefix $u$ and continues autoregressively. More generally, any constraint determined entirely by a bounded prefix can be enforced online without global renormalization.

At the opposite extreme are global constraints such as exact terminal position, exact syllable count, or general sentence-level optimization. For such events, exact conditional sampling at time $i$ requires probabilities of the form
\[
\Pr_{\theta}(E \mid x_{1:i-1}a),
\]
namely the probability that a partial continuation can still be completed into a full sequence satisfying the global event $E$. These are continuation masses over exponentially many suffixes. In the general autoregressive setting, the model class supplies no generic bounded sufficient statistic for these quantities, which is why exact conditioning becomes hard.

This also clarifies the role of regular-language constraints. If the model itself has a bounded hidden state, as in a finite-state Markov model, then a regular constraint can be combined with the model automaton and handled by dynamic programming on the product state space. In that bounded-state setting, exact sampling and exact MAP decoding under regular constraints are tractable. But for a general autoregressive model, the automaton state of the regular language is not enough: one would also need a bounded summary of the model's dependence on the entire prefix. Our hardness results show that no such general reduction exists in the broad succinct autoregressive setting.

Thus the right boundary is not simply between constrained and unconstrained generation. It is between constraints that admit an exact bounded-state recursion together with the model, and constraints that require globally summing or optimizing over future continuations.

\section{Inpainting as mixed prefix-suffix conditioning}

Inpainting provides another natural illustration of the same phenomenon. Filling a blank in a partially specified sequence amounts to sampling under both a prefix constraint and a suffix constraint. The prefix part is compatible with autoregressive generation: conditioning on a fixed prefix simply means starting generation from that prefix. The suffix part is different. It amounts to conditioning on a future event, and therefore requires exact continuation masses over all ways of completing the missing segment so as to match the prescribed suffix.

This perspective also helps interpret existing infilling and inpainting systems. In symbolic music, closely related tasks include future-aware constrained generation and musical inpainting systems such as Anticipation-RNN, score inpainting models, the Piano Inpainting Application, and DeepBach \cite{hadjeres-nielsen-2020,pati-lerch-hadjeres-2019,hadjeres-crestel-2021,hadjeres-pachet-nielsen-2017}. On the text side, fill-in-the-blank and text-infilling formulations likewise condition on both left and right context rather than only on a causal prefix \cite{donahue-lee-liang-2020}.

Formally, let a partially specified sequence be of the form
\[
u\, y\, v,
\]
where $u$ is a fixed prefix, $v$ a fixed suffix, and $y$ the missing segment to be generated. Exact inpainting requires sampling from the conditional distribution
\[
P_\theta(y \mid u\, y\, v \in \mathcal{C}).
\]

The fixed-length terminal problem studied above is an immediate special case: take $u=\epsilon$ empty, let $v=\mathtt{eos}$, and require the total length to be $L$. Then inpainting reduces exactly to sampling from
\[
P_\theta(x \mid x\in\mathcal{C},\ |x|=L),
\]
whose normalization constant was shown above to be \#P-hard to compute.

\begin{corollary}\label{cor:inpainting-hard}
Any generic exact inpainting procedure for succinct autoregressive models that samples from the true conditional distribution
\[
P_\theta(y \mid u\, y\, v \in \mathcal{C})
\]
by causal left-to-right renormalization of a fixed pretrained autoregressive model must compute \#P-hard conditioned normalization terms. In particular, exact inpainting does not reduce to ordinary left-to-right sampling from a fixed pretrained causal model by polynomial-time local reweighting alone.
\end{corollary}

\begin{proof}
Take the special case $u=\epsilon$, $v=\mathtt{eos}$, and fixed total length $L$. Then exact inpainting contains, as a special case, conditioning on the regular language of complete sequences of length exactly $L$ discussed in Section~\ref{sec:fixedlength}. The associated normalization constant is $Z_L(\theta)$, whose exact computation is \#P-hard by Theorem~\ref{thm:zl-hard}. Therefore any causal exact sampling scheme based on local renormalization of next-token probabilities evaluates \#P-hard continuation masses in the general case.
\end{proof}

This gives a clean interpretation of current inpainting proposals. Fast practical systems avoid the exact-conditioning problem rather than solve it head-on. Anticipation-RNN explicitly replaces exact Bayesian conditioning of a fixed unconstrained model by a separately trained constrained model; text-infilling systems likewise train or fine-tune models specifically for infilling; the Piano Inpainting Application uses a dedicated encoder-decoder architecture; and DeepBach uses pseudo-Gibbs sampling rather than exact causal conditioning \cite{hadjeres-nielsen-2020,donahue-lee-liang-2020,hadjeres-crestel-2021,hadjeres-pachet-nielsen-2017}. These methods are approximate inference procedures or changes of model class, not generic exact samplers for the conditional distribution induced by an arbitrary pretrained autoregressive model. Equivalently, with the original causal model held fixed, they produce biased samples relative to the exact conditioned target law.

\section{Discussion}

The complexity statements above have a natural interpretation for contemporary language models. Exact rhyme, exact syllable count, and positional letter or suffix constraints are all examples of global form constraints on complete sequences. A request such as ``produce a line of exactly twelve syllables,'' ``make the final word rhyme with \emph{time},'' ``ensure that the last word ends in \emph{-ing},'' or ``make the third word end with the letter \emph{e}'' is not hard to state, but exact satisfaction requires reasoning about future completions rather than only about local next-token likelihood.

This explains a familiar empirical phenomenon: autoregressive language models often come close to satisfying such constraints but miss by one syllable, produce an approximate rather than exact rhyme, or satisfy one positional requirement only by violating another. Such failures are consistent with the complexity results proved here. In the general autoregressive setting, exact enforcement of global formal constraints requires access to continuation masses or global optimization procedures that are computationally intractable in the worst case.

The results do not imply that language models never satisfy rhyme, meter, or positional constraints. They do so approximately, and specialized external machinery helps substantially: constrained decoding, finite-state controllers, reranking, symbolic metrical modules, backtracking search, or dedicated infilling architectures in text and music \cite{hadjeres-nielsen-2020,pati-lerch-hadjeres-2019,hadjeres-crestel-2021,donahue-lee-liang-2020,hadjeres-pachet-nielsen-2017} all impose structure that plain left-to-right sampling does not provide on its own. This paper stresses that exact satisfaction of such constraints does not arise generically from unconstrained local autoregressive sampling alone. When rhyme, meter, terminal patterns, or inpainting constraints are enforced without exact continuation masses, the resulting procedure is biased relative to the exact conditional distribution of the fixed underlying autoregressive model. These are hidden inferential biases: distortions that come from approximate conditioning.

\section{Conclusion}

Autoregressive factorization provides local access to next-token probabilities, not a tractable general procedure for global mode-finding or exact conditioning under global formal constraints. For succinctly represented autoregressive models with polynomial-time next-token evaluation, exact sentence-level MAP decoding is NP-hard, exact constrained MAP remains NP-hard under unary and metrical constraints, and exact conditioned normalization is already \#P-hard even for regular constraints such as fixed-length terminal events. The positive side of the picture is the first-order Markov case: in such a case, constrained generation can be formulated exactly, and exact sampling with the correct conditional probabilities is available for regular, unary, and Markov constraints \cite{pachet-roy-2026,papadopoulos-pachet-roy-sakellariou-2015}. Outside that bounded-state setting, the situation changes qualitatively. Whenever generation must satisfy exact global form constraints (rhyme, meter, terminal patterns, positional endings, or related formal devices) plain autoregressive generation faces a structural computational obstacle rather than a mere implementation inconvenience. Consequently, constrained uses of autoregressive models generally provide incomplete coverage of the admissible solution space and biased probabilities over the solutions they return.


\begin{thebibliography}{9}

\bibitem{eikema-aziz-2020}
Bryan Eikema and Wilker Aziz.
\newblock Is MAP Decoding All You Need? The Inadequacy of the Mode in Neural Machine Translation.
\newblock In \emph{Proceedings of COLING 2020}, pages 3080--3090, 2020.

\bibitem{blodgett-etal-2020-language}
Su Lin Blodgett, Solon Barocas, Hal Daum\'e III, and Hanna Wallach.
\newblock Language (Technology) is Power: A Critical Survey of ``Bias'' in NLP.
\newblock In \emph{Proceedings of the 58th Annual Meeting of the Association for Computational Linguistics}, pages 5454--5476, 2020.

\bibitem{meister-vieira-cotterell-2020}
Clara Meister, Tim Vieira, and Ryan Cotterell.
\newblock If Beam Search is the Answer, What was the Question?
\newblock In \emph{Proceedings of EMNLP 2020}, pages 2173--2185, 2020.

\bibitem{stahlberg-byrne-2019}
Felix Stahlberg and Bill Byrne.
\newblock On NMT Search Errors and Model Errors: Cat Got Your Tongue?
\newblock In \emph{Proceedings of EMNLP-IJCNLP 2019}, pages 3356--3362, 2019.

\bibitem{hadjeres-nielsen-2020}
Ga\"etan Hadjeres and Frank Nielsen.
\newblock Anticipation-RNN: enforcing unary constraints in sequence generation, with application to interactive music generation.
\newblock \emph{Neural Computing and Applications}, 32:995--1005, 2020.

\bibitem{pati-lerch-hadjeres-2019}
Ashis Pati, Alexander Lerch, and Ga\"etan Hadjeres.
\newblock Learning to Traverse Latent Spaces for Musical Score Inpainting.
\newblock In \emph{Proceedings of the 20th International Society for Music Information Retrieval Conference}, pages 343--351, 2019.

\bibitem{hadjeres-crestel-2021}
Ga\"etan Hadjeres and L\'eopold Crestel.
\newblock The Piano Inpainting Application.
\newblock arXiv:2107.05944, 2021.

\bibitem{donahue-lee-liang-2020}
Chris Donahue, Mina Lee, and Percy Liang.
\newblock Enabling Language Models to Fill in the Blanks.
\newblock In \emph{Proceedings of the 58th Annual Meeting of the Association for Computational Linguistics}, pages 2492--2501, 2020.

\bibitem{hadjeres-pachet-nielsen-2017}
Ga\"etan Hadjeres, Fran\c{c}ois Pachet, and Frank Nielsen.
\newblock DeepBach: a Steerable Model for Bach Chorales Generation.
\newblock In \emph{Proceedings of the 34th International Conference on Machine Learning}, pages 1362--1371, 2017.

\bibitem{pachet-roy-2026}
Fran\c{c}ois Pachet and Pierre Roy.
\newblock Markov Constraints for Controlled Sequence Generation.
\newblock To appear in \emph{Festum-PI 2025 Proceedings}, \emph{Lecture Notes in Mathematics}, Springer, N. Alikakos and C\'edric Villani, editors, 2026.

\bibitem{papadopoulos-pachet-roy-sakellariou-2015}
Alexandre Papadopoulos, Fran\c{c}ois Pachet, Pierre Roy, and Jason Sakellariou.
\newblock Exact Sampling for Regular and Markov Constraints with Belief Propagation.
\newblock In \emph{Principles and Practice of Constraint Programming}, pages 341--350, 2015.

\end{thebibliography}
\end{document}